\documentclass{article}

\pdfoutput=1

 \usepackage[preprint]{nips_2018}

\usepackage[utf8]{inputenc} 
\usepackage[T1]{fontenc}    
\usepackage{hyperref}       
\usepackage{url}            
\usepackage{booktabs}       
\usepackage{amsfonts}       
\usepackage{nicefrac}       
\usepackage{microtype}      

\usepackage{amsmath}
\usepackage{amssymb}
\usepackage{amsthm}
\usepackage{enumerate}
\usepackage{mathtools}
\usepackage{xcolor}

\usepackage{hyperref}

\usepackage{float}
\usepackage{subfig}

\usepackage{natbib}
\setcitestyle{numbers,square,sort&compress,comma}
\usepackage{algorithm}
\usepackage{algpseudocode}

\makeatletter
\let\OldStatex\Statex
\renewcommand{\Statex}[1][3]{%
	\setlength\@tempdima{\algorithmicindent}%
	\OldStatex\hskip\dimexpr#1\@tempdima\relax}
\makeatother

\makeatletter
\newcommand{\StatexIndent}[1][3]{%
	\setlength\@tempdima{\algorithmicindent}%
	\Statex\hskip\dimexpr#1\@tempdima\relax}
\algdef{S}[IF]{IfNoThen}[1]{\algorithmicif\ #1}%
\makeatother

\newtheorem{theorem}{Theorem}[section]
\newtheorem{lemma}{Lemma}[section]
\newtheorem{remark}{Remark}[section]

\newtheorem*{theorem*}{Theorem}
\newtheorem*{lemma*}{Lemma}

\newcommand{\p}{\partial}

\newcommand{\real}{\mathbb{R}}
\newcommand{\sym}[1]{\mathbb{S}^{#1}}

\newcommand{\postol}{{\epsilon}}
\newcommand{\maxtol}{{c}}
\newcommand{\sympos}[1]{\mathbb{S}^{#1}_\postol}

\newcommand{\schur}{\mathcal{S}}
\newcommand{\stabilizable}{\bar{\mathcal{S}}}

\newcommand{\norm}[2]{\Vert#1\Vert_{#2}}
\newcommand{\normal}[2]{\mathcal{N}\left({#1},{#2}\right)}
\newcommand{\normalsmall}[2]{\mathcal{N}({#1},{#2})}
\newcommand{\ev}{\mathbb{E}}
\newcommand{\trace}[1]{\textup{tr} \ #1}
\newcommand{\diag}{\textup{diag}}
\newcommand{\conv}{\textup{conv}}
\newcommand{\tpltz}{\textup{toeplitz}}

\newcommand{\ml}{\prec}
\newcommand{\mle}{\preceq}
\newcommand{\mg}{\succ}
\newcommand{\mge}{\succeq}

\newcommand{\st}{\textup{s.t.}}
\newcommand{\defeq}{:=}
\newcommand{\eqdef}{=:}
\newcommand{\refsec}{\S}
\newcommand{\mysec}{Section}

\newcommand{\cl}{{\textup{cl}}}

\newcommand{\disccol}{black} 
\newcommand{\jackcol}{black} 

\newcommand{\discuss}[1]{\textcolor{\disccol}{#1}} 
\newcommand{\jack}[1]{\textcolor{\jackcol}{#1}} 

\newcommand{\bw}{G}

\newcommand{\atr}{{A_\textup{tr}}}
\newcommand{\btr}{{B_\textup{tr}}}

\newcommand{\als}{A_\textup{ls}}
\newcommand{\bls}{B_\textup{ls}}

\newcommand{\distcov}{\Pi}
\newcommand{\distcovtr}{{\distcov_\textup{tr}}}

\newcommand{\param}{\theta}
\newcommand{\paramtr}{{\theta_\textup{tr}}}
\newcommand{\groupparam}[1]{\lbrace#1\rbrace}

\newcommand{\paramab}{\theta_{AB}}

\newcommand{\data}{\mathcal{D}}
\newcommand{\numrollouts}{N}
\newcommand{\roli}{r}
\newcommand{\rolis}{{\roli}}

\newcommand{\numgendata}{N}

\newcommand{\xplus}{x_{+}}
\newcommand{\xminus}{x_{-}}
\newcommand{\xu}{D_{x,u}}

\newcommand{\policy}{\phi}
\newcommand{\cost}{J} 

\newcommand{\costbound}{\hat{J}}
\newcommand{\mccostbound}{\hat{J}_\nummc^\conf}
\newcommand{\ccost}{J^c} 
\newcommand{\mccost}{J^c_\nummc} 

\newcommand{\kpsep}{|} 

\newcommand{\lyap}{X}

\newcommand{\conf}{c}
\newcommand{\confregion}{{\Theta^\conf}}

\newcommand{\invlyap}{Y}
\newcommand{\cntrllyap}{L}
\newcommand{\slacklyap}{Z}

\newcommand{\taylor}{T}
\newcommand{\nomx}{\bar{X}}
\newcommand{\nomlyap}{\bar{X}}
\newcommand{\nomk}{\bar{K}}

\newcommand{\klqr}{K_\textup{lqr}}

\newcommand{\convtol}{\epsilon}

\newcommand{\posterior}{\pi}
\newcommand{\uposterior}{\bar{\pi}}
\newcommand{\nummc}{M}

\newcommand{\sampleset}{\tilde{\Theta}^\conf_\nummc}

\newcommand{\ssc}[1]{_{#1}}

\newcommand{\wishart}[2]{\mathcal{W}^{-1}(#1,#2)}
\newcommand{\wishmat}{\Phi}
\newcommand{\wishdof}{\nu}

\newcommand{\opt}{\emph{optimal}}
\newcommand{\nom}{\emph{nominal}}
\newcommand{\wc}{\emph{worst-case}}
\newcommand{\mixed}{$H_2/H_\infty$}
\newcommand{\mixeditr}{\emph{alternate-r}}
\newcommand{\comlyap}{\emph{CL}}
\newcommand{\ours}{\emph{proposed}}
\newcommand{\mixeds}{\emph{alternate-s}}

\newcommand{\arm}{\alpha_a}
\newcommand{\pend}{\alpha_p}

\title{Learning convex bounds for linear quadratic control policy synthesis
}

\author{
  Jack Umenberger \\
  Department of Information Technology\\
  Uppsala University\\
  Sweden \\
  \texttt{jack.umenberger@it.uu.se} \\
   \And
   Thomas B. Sch\"on \\
  Department of Information Technology\\
	Uppsala University\\
	Sweden \\
\texttt{thomas.schon@it.uu.se} \\
}

\begin{document}

\maketitle

\begin{abstract}
Learning to make decisions from observed data in dynamic environments remains a problem of fundamental importance in a number of fields, from artificial intelligence and robotics, to medicine and finance.
This paper concerns the problem of learning control policies for unknown linear dynamical systems so as to maximize a quadratic reward function.
We present a method to optimize the expected value of the reward over the posterior distribution of the unknown system parameters, given data.
The algorithm involves sequential convex programing, and enjoys reliable local convergence and robust stability guarantees.
Numerical simulations and stabilization of a real-world inverted pendulum are used to demonstrate the approach, with strong performance and robustness properties observed in both.
\end{abstract}

\section{Introduction}
Decision making for dynamical systems in the presence of uncertainty is a problem of great prevalence and importance, 
as well as considerable difficulty, especially when knowledge of the dynamics is available only via limited observations of system behavior.
In machine learning, the data-driven search for a control policy to maximize the expected reward attained by a stochastic 
dynamic process is known as \emph{reinforcement learning} (RL) \cite{sutton1998reinforcement}.
Despite remarkable recent success in games \cite{mnih2015human,silver2016mastering}, a major obstacle to the deployment RL-based control on physical systems (e.g. robots and self-driving cars) is the issue of \emph{robustness}, i.e., guaranteed safe and reliable operation.
With the necessity of such guarantees widely acknowledged \cite{amodei2016concrete}, so-called `safe RL' remains an active area of research \cite{garcia2015comprehensive}.

The problem of robust automatic decision making for uncertain dynamical systems has also been the subject of intense study in the area of \emph{robust control} (RC) \cite{zhou1998essentials}.
In RC, one works with a set of plausible models and seeks a control policy that is guaranteed to stabilize all models within the set. 
In addition, there is also a performance objective to optimize, i.e. a reward to be maximized, or equivalently, a cost to be minimized. 
Such cost functions are usually defined with reference to either a nominal model \cite{doyle1994mixed,haddad1991mixed} or the worst-case model \cite{petersen2000minimax} in the set.
RC has been extremely successful in a number of engineering applications \cite{postlethwaite2007robust}; however,
as has been noted, e.g., \cite{tu2017least,ostafew2016robust}, robustness may (understandably) come at the expense of performance, particularly for worst-case design.

The problem we address in this paper lies at the intersection of reinforcement learning and robust control, and can be summarized as follows: given observations from an unknown dynamical system, we seek a policy to optimize the expected cost (as in RL), subject to certain robust stability guarantees (as in RC).
Specifically, we focus our attention on control of linear time-invariant dynamical systems, subject to Gaussian disturbances, with the goal of minimizing a quadratic function penalizing state deviations and control action.
When the system is known, this is the classical linear quadratic regulator (LQR), a.k.a. $H_2$, optimal control problem \cite{burl1998linear}. 
We are interested in the setting in which the system is unknown, and knowledge of the dynamics must be inferred from observed data.

\paragraph{Contributions and paper structure}
The principle contribution of this paper is an algorithm to optimize the expected value of the linear quadratic regulator reward/cost function, where the expectation is w.r.t. the posterior distribution of unknown system parameters, given observed data; c.f. \mysec\ \ref{sec:problem_formulation} for a detailed problem formulation.
Specifically, we construct a sequence of convex approximations (upper bounds) to the expected cost, that can be optimized via semidefinite programing \cite{vandenberghe1996semidefinite}.
The algorithm, developed in \mysec\ \ref{sec:solution}, invokes the majorize-minimization (MM) principle \cite{lange2000optimization}, and consequently enjoys reliable convergence to local optima.
An important part of our contribution lies in guarantees on the \discuss{robust stability properties of the resulting control policies}, c.f. \mysec\ \ref{sec:stability}.
We demonstrate the proposed method via two experimental case studies: i) the benchmark problem on simulated systems considered in \cite{dean2017sample,tu2017least}, and ii) stabilization of a real-world inverted pendulum.
Strong performance and robustness properties are observed in both.
\discuss{Moving forward, from a machine learning perspective this work contributes to the growing body of research concerned with ensuring robustness in RL, c.f. \mysec\ \ref{sec:literature}. 
From a control perspective, this work appropriates cost functions more commonly found in RL (namely, expected reward)  to a RC setting, with the objective of reducing conservatism of the resulting robust control policies.}

\section{Related work}\label{sec:literature}
Incorporating various notions of `robustness' into RL has long been an area of active research \cite{garcia2015comprehensive}.
In so-called `safe RL', one seeks to respect certain safety constraints during exploration and/or policy optimization, for example, avoiding undesirable regions of the state-action space \cite{geibel2005risk,abbeel2005exploration}.
A related problem is addressed in `risk-sensitive RL', in which the search for a policy takes both the expected value and variance of the reward into account \cite{mihatsch2002risk,depeweg2018decomposition}.
Recently, there has been an increased interest in notions of robustness more commonly considered in control theory, chiefly \emph{stability} \cite{ostafew2016robust,aswani2013provably}.
Of particular relevance is the work of \cite{berkenkamp2017}, which employs Lyapunov theory \cite{khalil1996noninear} to verify stability of learned policies.
Like the present paper, \cite{berkenkamp2017} adopts a Bayesian framework; however, \cite{berkenkamp2017} makes use of Gaussian processes \cite{rasmussen2004gaussian} to model the uncertain nonlinear dynamics, which are assumed to be deterministic.
A major difference between \cite{berkenkamp2017} and our work is the cost function; in the former the policy is selected by optimizing for worst-case performance, whereas we optimize the expected cost.
Robustness of data-driven control has also been the focus of a recently developed family of methods referred to as `coarse-ID control', c.f.,\cite{tu2017coarse,dean2017sample,boczar2018finite,simchowitz2018learning}, in which finite-data bounds on the accuracy of the least squares estimator are combined with modern robust control tools, such as \emph{system level synthesis} \cite{wang2016system}.
Coarse-ID builds upon so-called `$H_\infty$ identification' methods for learning models of dynamical systems, along with error bounds that are compatible with robust synthesis methods \cite{helmicki1991control,chen1993caratheodory,chen2000control}. 
$H_\infty$ identification assumes an adversarial (i.e. worst-case) disturbance model, whereas Coarse-ID is applicable to probabilistic models, such as those considered in the present paper.
Of particular relevance to the present paper is \cite{dean2017sample}, which provides sample complexity bounds on the performance of robust control synthesis for the infinite horizon LQR problem, when the true system is not known.
Such bounds necessarily consider the worst-case model, given the observed data, where as we are concerned with expected cost over the posterior distribution of models.
In closing, we briefly mention the so-called `Riemann-Stieltjes' class of optimal control problems, for uncertain continuous-time dynamical systems, c.f., e.g., \cite{ross2015riemann,ross2014unscented}. 
Such problems often arise in aerospace applications (e.g. satellite control) 
where the objective is to design an open-loop control signal (e.g. for an orbital maneuver) rather than a feedback policy.

\section{Problem formulation}\label{sec:problem_formulation}

In this section we describe in detail the specific problem that we address in this paper.
The following notation is used:
$\sym{n}$ denotes the set of $n\times n$ symmetric matrices;
$\sym{n}_{+}$ ($\sym{n}_{++}$) denotes the cone of positive semdefinite (positive definite) matrices.
$A\mge B$ denotes $A-B\in\sym{n}_{+}$, similarly for $\mg$ and $\sym{n}_{++}$.
The trace of $A$ is denoted $\trace{A}$. 
The transpose of $A$ is denoted $A'$.
$|a|_Q^2$ is shorthand for $a'Qa$.
The convex hull of set $\Theta$ is denoted $\conv\Theta$.
The set of Schur stable matrices is denoted $\schur$.

\paragraph{Dynamics, reward function and policies}
We are concerned with control of discrete linear time-invariant dynamical systems of the form
\begin{equation}\label{eq:trueSystem}
x_{t+1}=Ax_t + Bu_t + w_t, \qquad w_t\sim\normalsmall{0}{\distcov},
\end{equation}
where $x_t\in\real^{n_x}$, $u_t\in\real^{n_u}$, and $w_t\in\real^{n_w}$ denote the state, input, and unobserved exogenous disturbance at time $t$, respectively.
Let $\param\defeq\groupparam{A,B,\distcov}$.
Our objective is to design a feedback control policy $u_t=\policy(x_t)$ that minimizes the cost function
$\lim_{T\rightarrow\infty}\frac{1}{T}\sum_{t=0}^{T}\ev\left[x_t'Qx_t+u_t'Ru_t\right]$,
where $x_t$ evolves according to \eqref{eq:trueSystem}, and $Q\mge0$ and $R\mg0$ are user defined weight matrices.
A number of different parametrizations of the policy $\policy$ have been considered in the literature, from neural networks (popular in RL, e.g., \cite{berkenkamp2017}) to causal (typically linear) dynamical systems (common in RC, e.g., \cite{petersen2000minimax}).
In this paper, we will restrict our attention to static-gain policies of the form $u_t=Kx_t$, where $K\in\real^{n_u\times n_x}$ is constant.
As noted in \cite{dean2017sample}, controller synthesis and implementation, is simpler (and more computationally efficient) for such policies.
When the parameters of the true system, denoted $\paramtr\defeq\groupparam{\atr,\btr,\distcovtr}$, are known this is the infinite horizon LQR problem, the optimal solution of which is well-known \cite{bertsekas1995dp}.
We assume that $\paramtr$ is unknown; rather, our knowledge of the dynamics must be inferred from observed sequences of inputs and states.

\paragraph{Observed data} 
We adopt the 
\discuss{data-driven setup} 
used in \cite{dean2017sample}, and assume that $\data\defeq\lbrace x_{0:T}^\rolis, u_{0:T}^\rolis\rbrace_{\roli=1}^\numrollouts$ where  
$x_{0:T}^\rolis=\lbrace x_{t}^\rolis\rbrace_{t=0}^T$
is the observed state sequence attained by evolving the true system for $T$ time steps, starting from an arbitrary $x_0^\rolis$ and driven by arbitrary input
$u_{0:T}^\rolis=\lbrace u_{t}^\rolis\rbrace_{t=0}^T$.
Each of these $\numrollouts$ independent experiments is referred to as a \emph{rollout}.
We perform parameter inference in the offline/batch setting; i.e., all data $\data$ is assumed to be available at the time of controller synthesis.

\paragraph{Optimization objective}
Given observed data and, possibly, prior knowledge of the system, we then have the posterior distribution over the model parameters denoted $\posterior(\param)\defeq p(A,B,\distcov|\data)$, in place of the true parameters $\paramtr$.
The function that we seek to minimize is the expected cost w.r.t. the posterior distribution, i.e.,
\begin{equation}\label{eq:expected_cost}
\lim_{T\rightarrow\infty} \frac{1}{T} \sum_{t=0}^{T} \ev \left[x_t'Qx_t+u_t'Ru_t \ | \ x_{t+1}=Ax_t + Bu_t + w_t, \ w_t\sim\normal{0}{\distcov}, \ \lbrace A,B,\distcov\rbrace\sim \posterior(\param)\right].
\end{equation}
In practice, the support of $\posterior$ almost surely contains $\groupparam{A,B}$ that are unstabilizable, which implies that \eqref{eq:expected_cost} is infinite. 
Consequently, we shall consider averages over confidence regions w.r.t. $\posterior$.
For convenience, let us denote the infinite horizon LQR cost, for given system parameters $\param$, by
\begin{subequations}
\begin{align}
\cost(K\kpsep\param) \defeq& \lim_{t\rightarrow\infty} \ev\left[x_t'(Q+K'RK)x_t \ | \ x_{t+1}=(A+BK)x_t+ w_t,  \ w\sim\normal{0}{\distcov}\right] \\
=& \begin{cases} 
\trace{\lyap\distcov} \textup{ with } \lyap = (A+BK)'\lyap(A+BK)+Q+K'RK, & A+BK \in\schur \\
\quad\infty, & \textup{otherwise},
\end{cases}\label{eq:cost_lyap}
\end{align}
\end{subequations}
where the second equality follows from standard Gramian calculations, and $\schur$ denotes the set of Schur stable matrices.
As an alternative to \eqref{eq:expected_cost} we consider the cost function
$\ccost(K)\defeq\int_\confregion \cost(K\kpsep\param)\posterior(\param)d\param$,
where $\confregion$ denotes the $\conf\ \%$ confidence region of the parameter space w.r.t. the posterior $\posterior$.
\discuss{Though better suited to optimization than \eqref{eq:expected_cost}, which is almost surely infinite, this integral cannot be evaluated in closed form, due to the complexity of $\cost(\cdot\kpsep\param)$ w.r.t. $\param$.}
Furthermore, there is still no guarantee that $\confregion$ contain \emph{only} stabilizable models.
To circumvent both of these \jack{issues}, we \jack{propose} the following Monte Carlo approximation of $\ccost(K)$,
\begin{equation}\label{eq:mc_cost}
\mccost(K)\defeq\frac{1}{\nummc}{\sum}_{i=1}^\nummc \cost(K\kpsep\param\ssc{i}), \qquad \param\ssc{i}\sim\confregion\cap\stabilizable, \qquad i=1,\dots,\nummc,
\end{equation}
where $\stabilizable$ denotes the set of stabilizable $\groupparam{A,B}$.

\paragraph{Posterior distribution} Given data $\data$, the parameter posterior distribution is given by Bayes' rule:
\begin{equation}
\posterior(\param) \defeq p(\param|\data) 
=  \frac{1}{p(\data)}p(\data|\param)p(\param) 
\propto p(\param){\prod}_{\roli=1}^\numrollouts{\prod}_{t=1}^Tp(x_t^\rolis|x_{t-1}^\rolis,u_{t-1}^\rolis,\param)
\eqdef \uposterior(\param), 
\end{equation}
where $p(\param)$ denotes our prior belief on $\theta$, $p(x_t^\rolis|x_{t-1}^\rolis,u_{t-1}^\rolis,\param)=\normal{Ax_{t-1}^\rolis+Bu_{t-1}^\rolis}{\distcov}$, and $\uposterior=p(\data)\posterior$ denotes the unnormalized posterior.
\discuss{To sample from $\posterior$, we can distinguish between two different cases.
First, consider the case when $\distcovtr$ is known or can be reliably estimated independently of $\groupparam{A,B}$. 
This is the setting in, e.g., \cite{dean2017sample}. 
In this case, the likelihood can be equivalently expressed as a Gaussian distribution over $\groupparam{A,B}$.
Then, when the prior $p(A,B)$ is uniform (i.e. non-informative) or Gaussian (self-conjugate), the posterior $p(A,B|\distcovtr,\data)$ is also Gaussian, c.f. Appendix~\ref{sec:inference_known}.
Second, consider the general case in which $\distcovtr$, along with $\groupparam{A,B}$, is unknown.
In this setting, one can select from a number of methods adapted for Bayesian inference in dynamical systems, such as Metropolis-Hastings \cite{ninness2010bayesian}, Hamiltonian Monte Carlo \cite{cheung2009bayesian}, and Gibbs sampling \cite{ching2005bayesian,wills2012estimation}.
When one places a non-informative prior on $\distcov$ (e.g., $p(\distcov)\propto\det(\distcov)^{-\frac{n_x+1}{2}}$), each iteration of a Gibbs sampler targeting $\posterior$ requires sampling from either a Gaussian or an inverse Wishart distribution, for which reliable numerical methods exist; c.f., Appendix~\ref{sec:inference_unknown}.}
In both of these cases we can sample from $\posterior$ and evaluate $\uposterior$ point-wise.
To draw $\param\ssc{i}\sim\confregion\cap\stabilizable$, as in \eqref{eq:mc_cost}, we can first draw a large number of samples from $\posterior$, discard the (100$-\conf$)\% of samples with the lowest unnormalized posterior values, and then further discard any samples that happen to be unstabilizable.
For convenience, we define 
$\sampleset\defeq\lbrace \lbrace\param\ssc{i}\rbrace_{i=1}^\nummc \ : \ \param\ssc{i}\sim\confregion\cap\stabilizable, \ i=1,\dots,\nummc\rbrace$, which should be interpreted as a set of $\nummc$ realizations of this procedure for sampling $\param\ssc{i}\sim\confregion\cap\stabilizable$.

\paragraph{Summary} We seek the solution of the optimization problem $\min_K \ \mccost(K)$ for $K\in\real^{n_u\times n_x}$.

\section{Solution via semidefinite programing}\label{sec:solution}
In this section we present the principle contribution of this paper: a method for solving $\min_K \ \mccost(K)$ via convex (semidefinite) programing (SDP).
It is convenient to consider an equivalent representation 
\begin{subequations}\label{eq:mc_problem_alt}
\begin{align}
\min_{K, \ \lbrace\lyap\ssc{i}\rbrace_{i=1}^\nummc\in\sym{n_x}_{++}} \ & \frac{1}{\nummc}{\sum}_{i=1}^{\nummc} \trace{\lyap\ssc{i}\distcov\ssc{i}}, \label{eq:mc_problem_cost} \\
\st \quad\quad \ \ &  \lyap\ssc{i} \mge (A\ssc{i}+B\ssc{i}K)'\lyap\ssc{i}(A\ssc{i}+B\ssc{i}K)+Q+K'RK, \ \param\ssc{i}\in\sampleset, \label{eq:mc_problem_constraint}
\end{align}
\end{subequations}
where the Comparison Lemma \cite[Lecture 2]{oliveira2010course} has been used to replace the equality in \eqref{eq:cost_lyap} with the inequality in \eqref{eq:mc_problem_constraint}.
\discuss{We introduce the notation $\sympos{n}\defeq\lbrace S\in\sym{n}: S\mge\postol I, S\mle \maxtol I\rbrace$, where $\postol$ and $\maxtol$ are arbitrarily small and large positive constants, respectively.
$\sympos{n}$ serves as a compact approximation of $\sym{n}_{++}$, suitable for use with SDP solvers, i.e., $S\in\sympos{n}\implies S\in\sym{n}_{++}$.}

\subsection{Common Lyapunov relaxation}\label{sec:common_lyap}
The principle challenge in solving \eqref{eq:mc_problem_alt} is that the constraint \eqref{eq:mc_problem_constraint} is not jointly convex in $K$ and $\lyap^i$.
The usual approach to circumventing this nonconvexity is to first apply the Schur complement to  \eqref{eq:mc_problem_constraint}, and then conjugate by the matrix $\diag(\lyap\ssc{i}^{-1},I,I,I)$, which leads to the equivalent constraint
\begin{equation}\label{eq:schur_conj_constraint}
\left[\begin{array}{cccc}
\lyap\ssc{i}^{-1} 		& \lyap\ssc{i}^{-1}  (A\ssc{i}+B\ssc{i}K)'    & \lyap\ssc{i}^{-1}  Q^{1/2} & \lyap\ssc{i}^{-1}  K' \\
(A\ssc{i}+B\ssc{i}K)\lyap\ssc{i}^{-1} 	  & \lyap\ssc{i}^{-1}  & 0 \\
Q^{1/2}\lyap\ssc{i}^{-1}         &  0 & I & 0 \\
K\lyap\ssc{i}^{-1}  & 0 & 0 & R^{-1}  
\end{array}\right]\mge 0.
\end{equation}
With the change of variables $\invlyap\ssc{i}=\lyap\ssc{i}^{-1}$ and $\cntrllyap\ssc{i}=K\lyap\ssc{i}^{-1}$, \eqref{eq:schur_conj_constraint} becomes an LMI, in $\invlyap\ssc{i}$ and $\cntrllyap\ssc{i}$.
This approach is effective when $\nummc=1$ (i.e. we have a single nominal system, as in standard LQR).
However, when $\nummc>1$ we cannot introduce a new $\invlyap\ssc{i}$ for each $\lyap\ssc{i}^{-1}$, as we lose uniqueness of the controller $K$ in $\cntrllyap\ssc{i}=K\lyap\ssc{i}^{-1}$, i.e., in general $\cntrllyap\ssc{i}\invlyap\ssc{i}^{-1}\neq\cntrllyap\ssc{j}\invlyap\ssc{j}^{-1}$ for $i\neq j$.
One strategy (\jack{prevalent in robust control}, 
e.g., \cite[\refsec C]{dean2017sample}) 
is to employ a `common Lyapunov function', i.e., $\invlyap=\lyap\ssc{i}^{-1}$ for all $i=1,\dots,\nummc$.
This gives the following convex relaxation (\jack{upper bound}) of problem \eqref{eq:mc_problem_alt},
\begin{subequations}\label{eq:common_lyap}
\begin{align}
&\min_{K, \ \invlyap\in\sympos{n_x}, \ \lbrace\slacklyap\ssc{i}\rbrace_{i=1}^\nummc\in\sym{n_x}} \ \trace{\slacklyap\ssc{i}}, \\
&\st \quad \left[\begin{array}{cc}
\slacklyap^i & \bw\ssc{i} \\ \bw\ssc{i}' & \invlyap
\end{array}\right]\mge 0, \ 
\left[\begin{array}{cccc}
\invlyap	& \invlyap A\ssc{i}' + \cntrllyap'B\ssc{i}'     & \invlyap Q^{1/2} & \cntrllyap'  \\
A\ssc{i}\invlyap + B\ssc{i}\cntrllyap 	  & \invlyap & 0 \\
Q^{1/2}\invlyap       &  0 & I & 0 \\
\cntrllyap  & 0 & 0 & R^{-1}  
\end{array}\right]\mge 0, \ \param\ssc{i}\in\sampleset,
\end{align}
\end{subequations}
where $\bw\ssc{i}$ denotes the Cholesky factorization of $\distcov\ssc{i}$, i.e., $\distcov\ssc{i}=\bw\ssc{i}\bw\ssc{i}'$, and $\lbrace\slacklyap^i\rbrace_{i=1}^\nummc$ are slack variables used to encode the cost \eqref{eq:mc_problem_cost} with the change of variables, i.e.,
\begin{equation*}
\min_\invlyap \trace{\invlyap^{-1}\distcov\ssc{i}} \leq \big\lbrace \min_{\invlyap,\slacklyap\ssc{i}}\trace{\slacklyap\ssc{i}} \ \ \st \ \ \slacklyap\ssc{i} \mge \bw\ssc{i}'\invlyap^{-1}\bw\ssc{i} \big\rbrace\iff
\min_{\invlyap,\slacklyap\ssc{i}} \trace{\slacklyap\ssc{i}} \ \st \
\left[\begin{array}{cc}
\slacklyap\ssc{i} & \bw\ssc{i} \\ \bw\ssc{i}' & \invlyap
\end{array}\right]\mge 0.
\end{equation*}
The approximation in \eqref{eq:common_lyap} is highly conservative, which motivates the iterative local optimization method presented in \mysec\ \ref{sec:convex_bound}.
Nevertheless, \eqref{eq:common_lyap} provides a principled way (i.e., a one-shot convex program) to initialize the iterative search method derived in \mysec~\ref{sec:convex_bound}.

\subsection{Iterative improvement by sequential semidefinite programing}\label{sec:convex_bound}

To develop this iterative search method first consider an equivalent representation of $\cost(K\kpsep\param\ssc{i})$, 
\begin{subequations}\label{eq:localCostAlt}
	\begin{align}
	\cost(K\kpsep\param\ssc{i}) = &\min_{\lyap\ssc{i}\in\sympos{n_x}} \trace \ \lyap\ssc{i}\distcov\ssc{i}\\
	& \quad \st \ \left[\begin{array}{ccc}
	\lyap\ssc{i}- Q & (A\ssc{i}+B\ssc{i}K)' & K' \\
	A\ssc{i}+B\ssc{i}K & \lyap\ssc{i}^{-1} & 0 \\
	K & 0 & R^{-1}
	\end{array}\right]\mge 0, \label{eq:localCost_nlmi}
	\end{align}
\end{subequations}
This representation highlights the nonconvexity of $\cost(K\kpsep\param\ssc{i})$ due to the $\lyap\ssc{i}^{-1}$ term, which was addressed (in the usual way) by a change of variables in \mysec~\ref{sec:common_lyap}.
In this section, we will instead replace $\lyap\ssc{i}^{-1}$ with a linear approximation and prove that this leads to a tight convex upper bound.
Given $S\in\sym{n}_{++}$, let $\taylor(S,S_0)$ denote the first order (i.e. linear) Taylor series approximation of $S^{-1}$ about some nominal $S_0\in\sym{n}_{++}$, i.e.,
$\taylor(S,S_0) \defeq S_0^{-1} + \left.\frac{\p S^{-1}}{\p S}\right|_{S=S_0}  \left(S-S_0\right) = S_0^{-1} - S_0^{-1}\left(S-S_0\right)S_0^{-1}$. 
We now define the function
\begin{subequations}\label{eq:localBound}
	\begin{align}
	\costbound(K,\nomk\kpsep\param\ssc{i}) \defeq &\min_{\lyap\ssc{i}\in\sympos{n_x}} \trace{\lyap\ssc{i}\distcov\ssc{i}} \\
	& \quad \st \ \left[\begin{array}{ccc}
	\lyap\ssc{i}- Q & (A\ssc{i}+B\ssc{i}K)' & K' \\
	A\ssc{i}+B\ssc{i}K & \taylor(\lyap\ssc{i},\nomlyap\ssc{i}) & 0 \\
	K & 0 & R^{-1}
	\end{array}\right]\mge 0,\label{eq:bound_lmi}
	\end{align}
\end{subequations}
\discuss{where $\nomlyap\ssc{i}$ is any $\lyap\ssc{i}\in\sympos{n_x}$ that achieves the minimum in \eqref{eq:localCostAlt}, with $K=\nomk$ for some nominal $\nomk$, i.e., $\cost(\nomk\kpsep\param\ssc{i})=\trace{\nomlyap\ssc{i}\distcov\ssc{i}}$.  }
Analogously to \eqref{eq:mc_cost}, we define 
\begin{equation}\label{eq:bound}
\mccostbound(K,\nomk)\defeq \frac{1}{\nummc} {\sum}_{\param\ssc{i}\in\sampleset} \costbound(K,\nomk\kpsep\param\ssc{i}).
\end{equation}
We now show that $\mccostbound(K,\nomk)$ is a convex upper bound on $\mccost(K)$, which is tight at $K=\nomk$.
The proof is given in \ref{sec:bound_proof} and makes use of the following technical lemma (c.f. \ref{sec:taylor_proof} for proof),
\begin{lemma}\label{lem:taylor}
	$\taylor(S,S_0) \mle S^{-1}$ for all $S,S_0\in\sym{n}_{++}$, where $\taylor(S,S_0)$ denotes the first-order Taylor series expansion of $S^{-1}$ about $S_0$ .
\end{lemma}

\begin{theorem}\label{thm:bound}
	Let $\mccostbound(K,\nomk)$ be defined as in \eqref{eq:bound}, \discuss{with $\nomk$ such that $\mccost(\nomk)$ is finite.} 
	Then $\mccostbound(K,\nomk)$ is a convex upper bound on $\mccost(K)$, i.e., $\mccostbound(K,\nomk)\geq\mccost(K) \ \forall K$.
	Furthermore, the bound is `tight' at $\nomk$, i.e., $\mccostbound(\nomk,\nomk)=\mccost(\nomk)$.
\end{theorem}

\paragraph{Iterative algorithm}
To improve upon the common Lyapunov solution given by \eqref{eq:common_lyap}, we can solve a sequence of convex optimization problems: $K^{(k+1)}=\arg\min_K \mccostbound(K,K^{(k)})$, c.f. Algorithm~\ref{alg:sequential_sdp} for details.
This procedure of optimizing tight surrogate functions in lieu of the actual objective function is an example of the `majorize-minimization (MM) principle', a.k.a. optimization transfer \cite{lange2000optimization}.
MM algorithms enjoy good numerical robustness, and (with the exception of some pathological cases) reliable convergence to local minima \cite{vaida2005parameter}.
Indeed, it is readily verified that $\mccost(K^{(k)})=\mccostbound(K^{(k)},K^{(k)})\geq\mccostbound(K^{(k+1)},K^{(k)})\geq \mccost(K^{(k+1)})$, where equality follows from tightness of the bound, and the second inequality is due to the fact that $\mccostbound(K,K^{(k)})$ is an upper bound.
This implies that $\lbrace \mccost(K^{(k)})\rbrace_{k=1}^\infty$ is a converging sequence.

\begin{algorithm}
	\caption{Optimization of $\mccost(K)$ via semidefinite programing}\label{alg:sequential_sdp}
	\begin{algorithmic}[1]
		\State \textbf{Input:} observed data $\data$, confidence $\conf$, LQR cost matrices $Q$ and $R$, number of particles in Monte Carlo approximation $\nummc$, convergence tolerance $\convtol$.
		\State Generate $\nummc$ samples from $\confregion\cap\stabilizable$, i.e., $\sampleset$, using the appropriate Bayesian inference method from \mysec~\ref{sec:problem_formulation}.
		\State Solve \eqref{eq:common_lyap}. Let $K_\cl$ denote the optimal solution of \eqref{eq:common_lyap}. Set $K^{(0)}\leftarrow \infty$, $K^{(1)}\leftarrow K_\cl$ and $k\leftarrow1$.
		\While{$|\mccost(K^{(k)}-\mccost(K^{(k-1)})|>\convtol$} \label{alg:convergence}
		\State Solve $K^*=\arg\min_K \ \mccostbound(K,K^{(k)})$. Set $K^{(k+1)}\leftarrow K^*$ and $k\leftarrow k+1$. 
		\EndWhile 
		\State\Return $K^{(k)}$ as the control policy.
	\end{algorithmic}
\end{algorithm}

\begin{remark}\label{rem:iterative}
	\textcolor{\disccol}{This sequential SDP approach can be applied in other robust control settings, e.g., mixed $H_2/H_\infty$ \cite{doyle1994mixed}, to improve on the common Lyapunov solution, c.f., \mysec \ref{sec:simulations} for an illustration.}
\end{remark}

\subsection{Robustness}\label{sec:stability}
\discuss{
Hitherto, we have considered the performance component of the robust control problem, namely minimization of the expected cost; we now address the robust stability requirement.
It is desirable for the learned policy to stabilize every model in the confidence region $\confregion$; in fact, this is necessary for the cost $\ccost(K)$ to be finite.
Algorithm \ref{alg:sequential_sdp} ensures stability of each of the $\nummc$ sampled systems from $\sampleset$, which implies that $\policy$ stabilizes the entire region as $\nummc\rightarrow\infty$.
However, we would like to be able to say something about robustness for finite $\nummc$.
To this end, we make two remarks.
}
First, if closed-loop stability of each sampled model is verified with a common Lyapunov function, then the policy stabilizes the convex hull of the sampled systems: 
\begin{theorem}\label{thm:convex_hull}
	Suppose there exists $K\in\real^{n_x\times n_u}$ such that $(A\ssc{i}+B\ssc{i}K)'X(A\ssc{i}+B\ssc{i}K)-X\ml0$ for $X\mg0$ and all $\Theta=\lbrace A\ssc{i},B\ssc{i}\rbrace_{i=1}^N$.
	Then $(A+BK)'X(A+BK)-X\ml0$ for all $\lbrace A,B\rbrace\in\conv\Theta$, where $\conv\Theta$ denotes the convex hull of $\Theta$.
\end{theorem}
The proof of Theorem \ref{thm:convex_hull} is given in \ref{sec:convex_hull_proof}.
The conditions of Theorem \ref{thm:convex_hull} hold for the common Lyapunov approach in \eqref{eq:common_lyap}, and can be made to hold for Algorithm \ref{alg:sequential_sdp} by introducing an additional Lyapunov stability constraint (with common Lyapunov function) for each sampled system, at the expense of some conservatism.
Second, we observe empirically that Algorithm \ref{alg:sequential_sdp} returns policies that very nearly stabilize the entire region $\confregion$, despite a very modest number of samples $\nummc$ relative to the dimension of the parameter space, c.f., \mysec \ref{sec:simulations}, in particular Figure \ref{fig:samples}.
A number of recent papers have investigated sampling (or grid) based approaches to stability verification of control policies, e.g., \cite{vinogradska2016stability,berkenkamp2017,bobiti2016sampling}.
Understanding why policies from Algorithm \ref{alg:sequential_sdp} generalize effectively to the entire region $\confregion$ is an interesting topic of future research.

\section{Experimental results}

\subsection{Numerical simulations using synthetic systems}\label{sec:simulations}
In this section, we study the infinite horizon LQR problem specified by
\begin{equation*}
\atr = \tpltz(a,a'), \ a = \left[1.01,0.01,0,\dots,0\right]\in\real^{n_x}, \ \btr = I, \ \distcovtr = I, \ Q =  10^{-3}I, \ R = I,
\end{equation*}
where $\tpltz(r,c)$ denotes the Toeplitz matrix with first row $r$ and first colum $c$. 
This is the same problem studied in 
\cite[\refsec 6]{dean2017sample} 
(for $n_x=3$), where it is noted that such dynamics naturally arise in \jack{consensus and distributed averaging problems.}
To obtain problem data $\data$, each \emph{rollout} involves simulating \eqref{eq:trueSystem}, with the true parameters, for $T=6$ time steps, excited by $u_t\sim\normal{0}{I}$ with $x_0=0$.
Note: to facilitate comparison with \cite{dean2017sample}, we too shall assume that $\distcovtr$ is known.
Furthermore, for all experiments $\confregion$ will denote the 95\% confidence region, as in \cite{dean2017sample}.
We compare the following methods of control synthesis:
\textbf{existing methods:}
(i) \nom: standard LQR using the nominal model from the least squares, i.e., $\lbrace \als,\bls\rbrace \defeq \arg\min_{A,B}{\sum}_{\roli=1}^\numrollouts{\sum}_{t=1}^{T}|x^\rolis_t-Ax^\rolis_{t-1}-Bu^\rolis_{t-1}|^2$;
(ii) \wc: optimize for worst-case model (95\% confidence) s.t. robust stability constraints, i.e., the method of \cite[\refsec 5.2]{dean2017sample};
(iii) \mixed: enforce stability constraint from \cite[\refsec 5.2]{dean2017sample}, but optimize performance for the nominal model $\lbrace \als,\bls\rbrace$;
\textbf{proposed method(s):}
(iv) \comlyap: the common Lyapunov relaxation of \ref{eq:common_lyap};
(v) \ours: the method proposed in this paper, i.e., Algorithm~\ref{alg:sequential_sdp};
\textbf{additional new methods:}
(vi) \mixeditr: initialize with the  $H_2/H_\infty$ solution, and apply the iterative optimization method proposed in \mysec~\ref{sec:convex_bound}, c.f., Remark \ref{rem:iterative};
(vii) \mixeds: optimize for the nominal model $\lbrace\als,\bls\rbrace$, enforce stability for the sampled systems in $\sampleset$. 
Before proceeding, we wish to emphasize that the different control synthesis methods have different objectives; a lower cost does not mean that the associated method is `better'. 
This is particularly true for \wc\ which seeks to optimize performance for the worst possible model so as to bound the cost on the true system.

To evaluate \textbf{performance}, we compare the cost of applying a learned policy $K$ to the true system $\paramtr=\lbrace\atr,\btr\rbrace$, to the optimal cost achieved by the optimal controller $\klqr$ (designed using $\paramtr$), i.e., $\cost(K\kpsep\paramtr)/\cost(\klqr\kpsep\paramtr)$.
We refer to this as `LQR suboptimality.'
In Figure \ref{fig:cost_infeas_nx3} we plot LQR suboptimality is shown as a function of the number of rollouts $\numrollouts$, for $n_x=3$.
We make the following observations.
Foremost, the methods that enforce stability `stochastically' (i.e. point-wise), namely \ours\ and \mixeds, attain significantly lower costs than the methods that enforce stability `robustly'.
Furthermore, in situations with very little data, e.g. $\numrollouts=5$, the robust control methods are usually unable to find a stabilizing controller, yet the \ours\ method finds a stabilizing controller in the majority of trials.
Finally, we note that the iterative procedure in \ours\ (and \mixeds) significantly improves on the common-Lyapunov relaxation \comlyap; similarly, \mixeditr\ consistently improves upon \mixed\ (as expected).
\begin{figure}[h]
	\centering
	\includegraphics[width=0.9\linewidth]{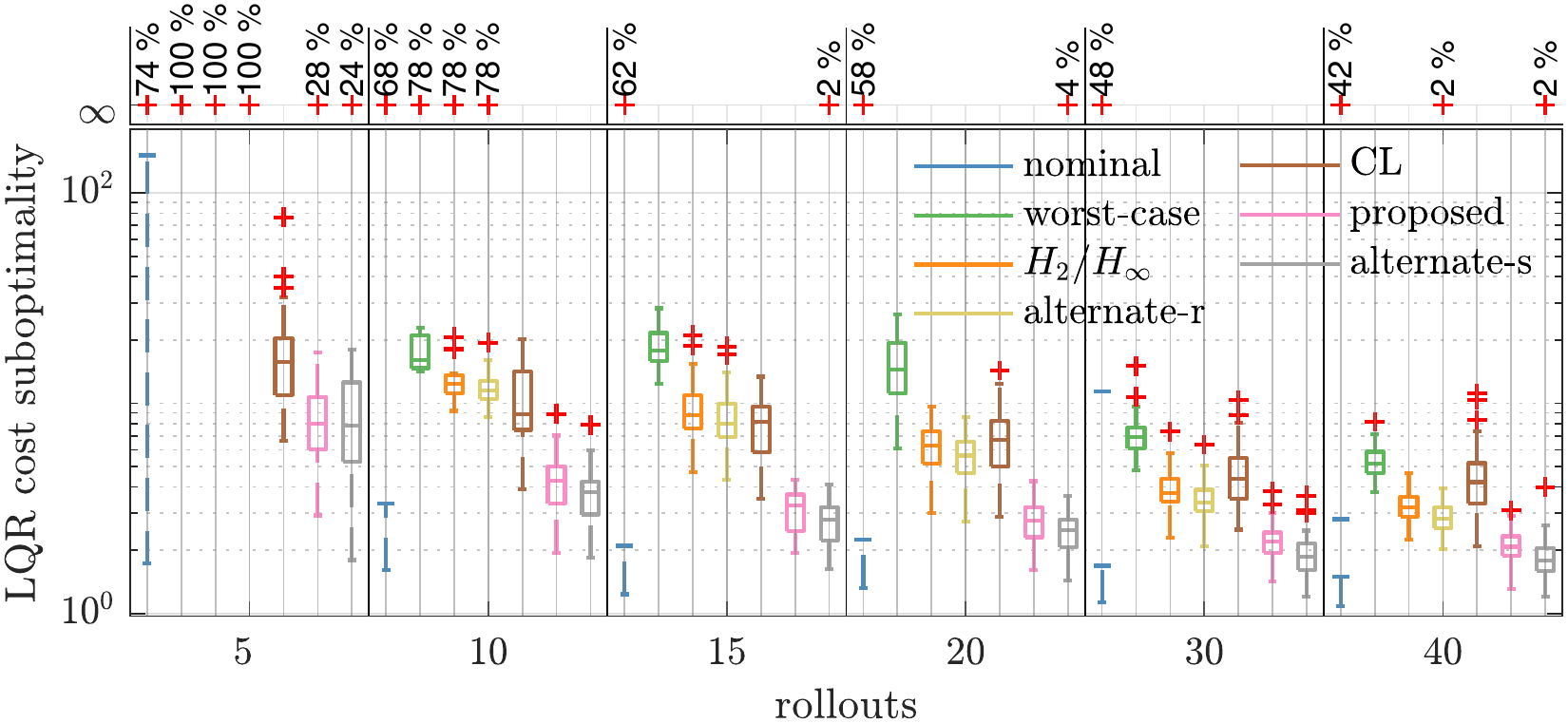}
	\caption{
LQR suboptimality as a function of the number of rollouts (i.e. amount of training data).
$\infty$ suboptimality denotes cases in which the method was unable to find a stabilizing controller for the \emph{true} system (including infeasibility of the optimization problem for policy synthesis),
and the \% denotes the frequency with which this occurred for the 50 experimental trials conducted.	
}
	\label{fig:cost_infeas_nx3}
\end{figure}

It is natural to ask whether the reduction in cost exhibited by \ours\ (and \mixeds) come at the expense of \textbf{robustness}, namely, the ability to stabilize a large region of the parameter space.
Empirical results suggest that this is \emph{not} the case.
To investigate this we sample 5000 fresh (i.e. not used for learning) models from $\confregion\cap\stabilizable$ and check closed-loop stability of each; this is repeated for 50 independent experiments with varying $n_x$ and $\numrollouts=50$. 
The median percentage of models that were \emph{unstable} in closed-loop is recorded in Table \ref{tab:stability_fixt}. 
We make two observations:
(i) the \ours\ method exhibits strong robustness. 
Even for $n_x=12$ (i.e. 288-dim parameter space), it stabilizes $>99$\% of samples from the confidence region, with only $\nummc=100$ MC samples.
(ii) when the robust methods (\wc, \mixed, \mixeditr) are feasible, the resulting policies were found to stabilize 100\% of samples; however, for $n_x=12$, the methods were infeasible almost half the time, whereas \ours\ always returned a policy.
Further evidence is provided in Figure \ref{fig:samples}, which plots robustness and performance as a function of the number of MC samples, $\nummc$.
For $n_x=3$ and $\nummc\geq800$, the entire confidence region is stabilized with very high probability, suggesting that $\nummc\rightarrow\infty$ is not required for robust stability in practice.

\begin{table}[t]
	\caption{Median \% of unstable closed-loop models, with open-loop models sampled from the 95\% confidence region of the posterior, for system of varying dimension $n_x$; c.f. \mysec~\ref{sec:simulations} for details.
	Parenthesized quantities denote the \% of cases for which the policy synthesis optimization problem was infeasible (i.e. no policy was returned).
	50 experiments were conducted, with $\numrollouts=50$.
	\mixed\ and \mixeditr\ have the same robustness guarantees as \wc, and are omitted.}
	\label{tab:stability_fixt}
	\centering
	\begin{tabular}{llllllll}
	\toprule
	$n_x$ & \opt & \nom & \wc  & \comlyap & \textbf{\ours}  & \mixeds \\
	\midrule
	3     &  61.6 (0) & 28.75 (0)  & 0 (0) & 0 (0) & 0.10 (0) & 1.35 (0)   \\
	6     &  95.37 (0) & 58.41 (0) & 0 (0) & 0 (0) & 0.18 (0) & 1.76 (0)   \\
	9     &  99.6 (0) & 81.9 (0)  & 0 (0) & 0 (0) & 0.24 (0) & 1.40 (0)   \\
	12   &  100 (0) & 94.28 (0) & 0 (46.0) & 0 (0) & 0.27 (0) & 1.27 (0)   \\
	\bottomrule
\end{tabular}
\end{table}

\begin{figure}
	\centering
		\includegraphics[width=0.9\linewidth]{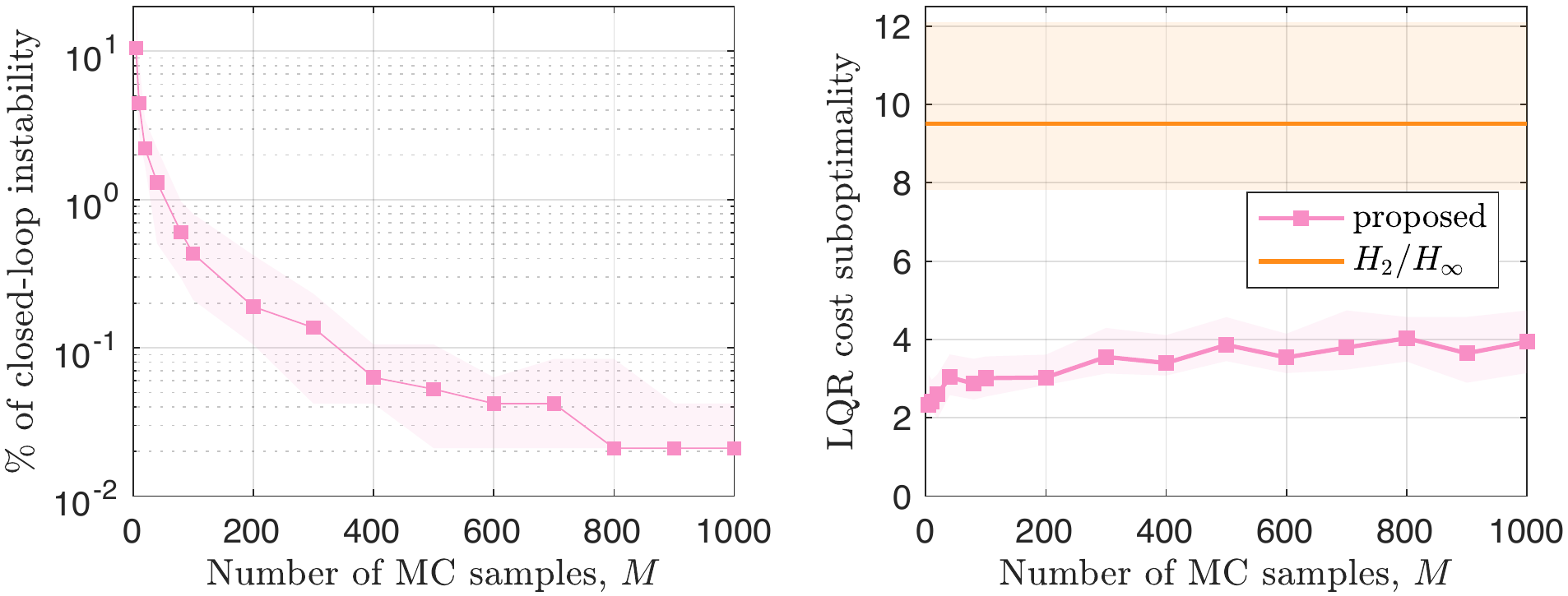}
	\caption{(left) Median \% of unstable closed-loop models, with open-loop models sampled from the 95\% confidence region of the posterior, for $n_x=3$ and $\numrollouts=15$, as a function of the number of samples $\nummc$ used in the MC approximation \eqref{eq:mc_cost}. (right) LQR suboptimality as a function of $\nummc$. 
	50 experiments were conducted, c.f. \mysec \ref{sec:simulations} for details. 
	Shaded regions cover the interquartile range.}
	\label{fig:samples}
\end{figure}

\subsection{Real-world experiments on a rotary inverted pendulum}\label{sec:pendulum}
We now apply the proposed algorithm to the classic problem of stabilizing a (rotary) inverted pendulum, on real (i.e. physical) hardware (Quanser QUBE 2), c.f. \ref{sec:pendulum_details} for details.
To generate training data, the superposition of a non-stabilizing control signal and a sinusoid of random frequency is applied to the rotary arm motor while the pendulum is inverted.
The arm and pendulum angles (along with velocities) are sampled at 100Hz until the pendulum angle exceeds $20^\circ$, which takes no more than 5 seconds. 
This constitutes one rollout.
We applied the \wc, \mixed, and \ours\ methods to optimize the LQ cost with $Q=I$ and $R=1$.
To generate bounds $\epsilon_A\geq\norm{\als-\atr}{2}$ and $\epsilon_B\geq\norm{\bls-\btr}{2}$ for \wc\ and \mixed, we sample 
$\groupparam{A\ssc{i},B\ssc{i}}_{i=1}^{5000}$ 
from the 95\% confidence region of the posterior, using Gibbs sampling, and take $\epsilon_A=\max_i \norm{\als-A\ssc{i}}{2}$ and $\epsilon_B=\max_i \norm{\bls-B\ssc{i}}{2}$.
The \ours\ method used 100 such samples for synthesis.
We also applied the \emph{least squares policy iteration} method \cite{lagoudakis2003least}, but none of the policies could stabilize the pendulum given the amount of training data.
Results are presented in Figure \ref{fig:pendulum_costs}, from which we make the following remarks.
First, as in \mysec \ref{sec:simulations}, the \ours\ method achieves high performance (low cost), especially in the low data regime where the magnitude of system uncertainty renders the other synthesis methods infeasible.
Insight into this performance is offered by Figure \ref{fig:pendulum_costs}(b), which indicates that policies from the \ours\ method stabilize the pendulum with control signals of smaller magnitude.
Finally, performance of the \ours\ method converges after very few rollouts.
Data-inefficiency is a well-known limitation of RL; understanding and mitigating this inefficiency is the subject of considerable research \cite{dean2017sample,tu2017least,deisenroth2011pilco,schulman2015high,gu2016q,gu2016continuous}.
Investigating the role that a Bayesian approach to uncertainty quantification plays in the apparent sample-efficiency of the proposed method is an interesting topic for further inquiry.

\begin{figure}
	\centering
	\subfloat[]{
		\includegraphics[width=0.45\linewidth]{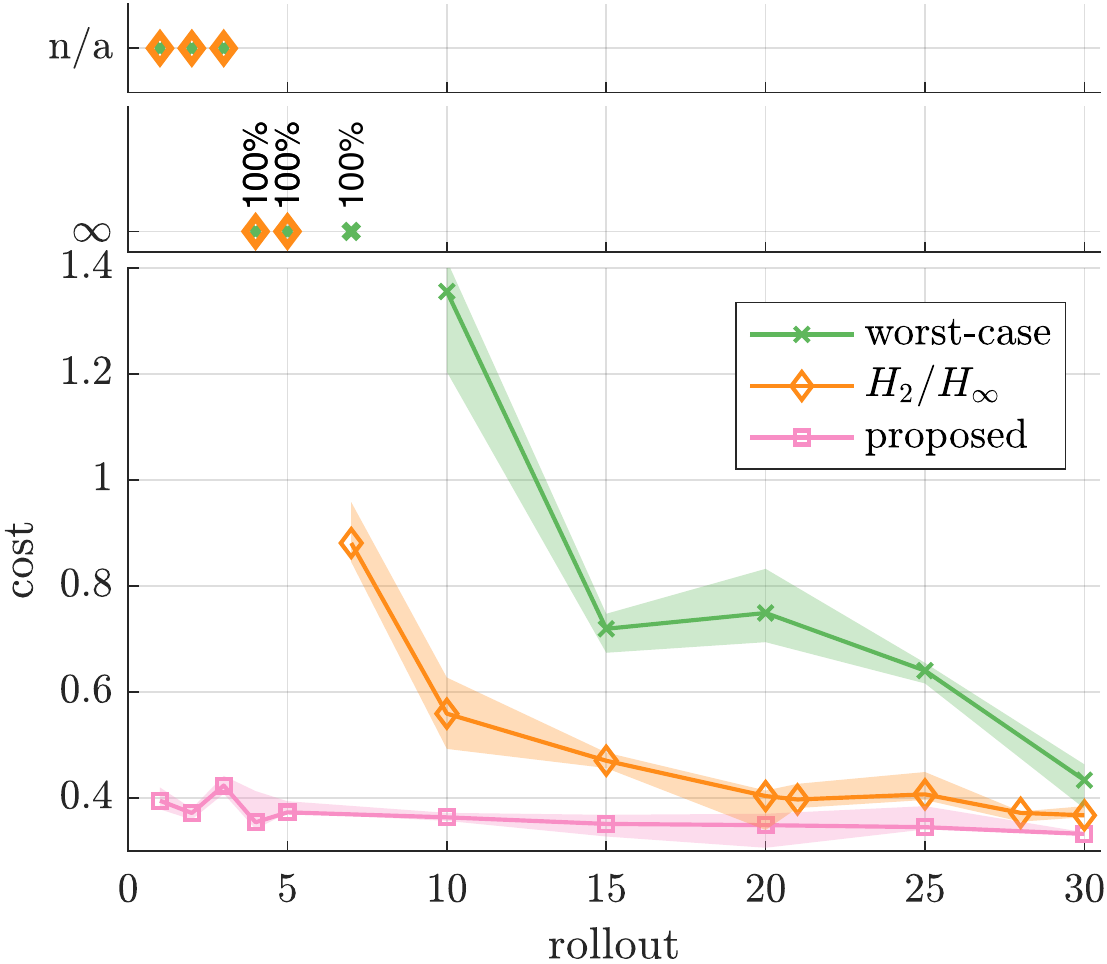}} 
	\subfloat[]{
		\includegraphics[width=0.45\linewidth]{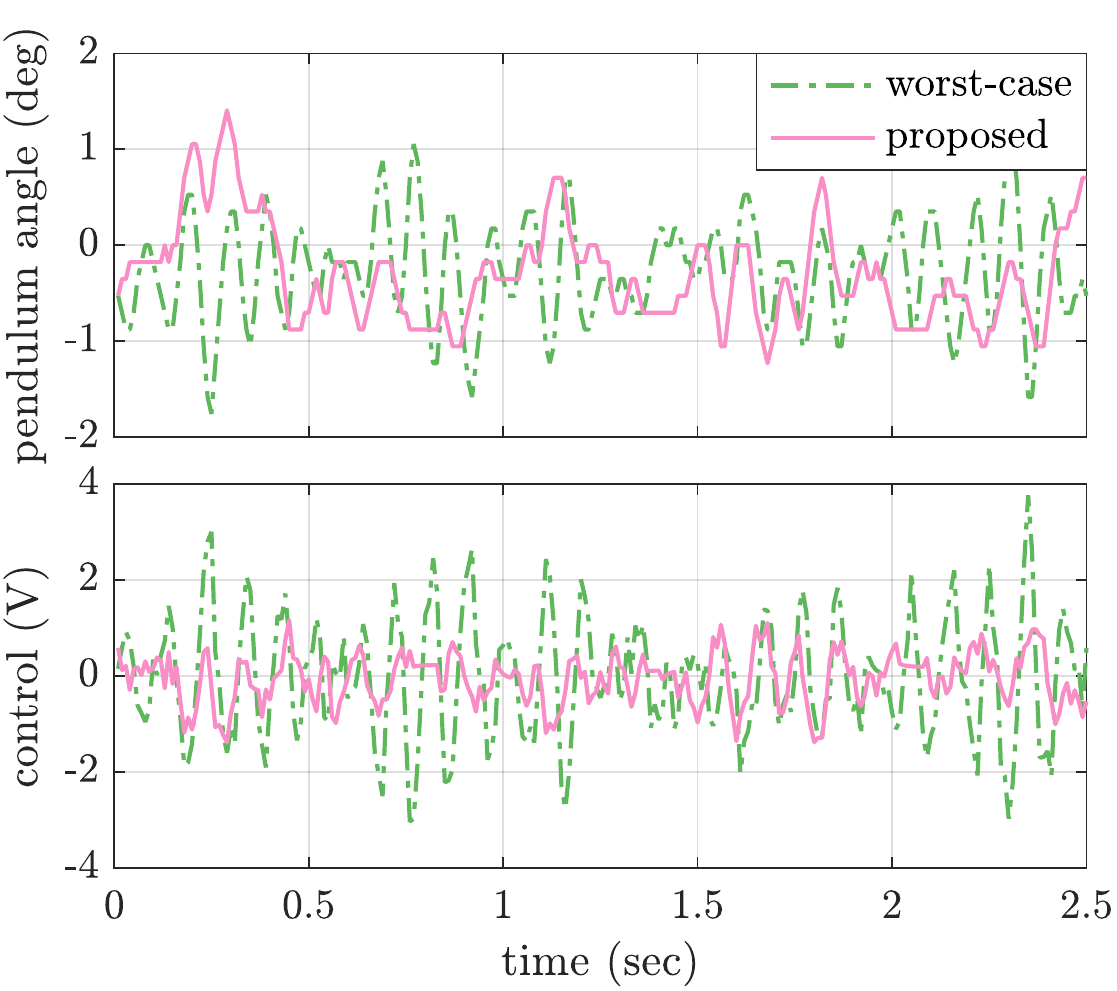}}
	\caption{(a) LQR cost on real-world pendulum experiment, as a function of the number of rollouts. $\infty$ cost denotes controllers that resulted in instability during testing. 
		n/a denotes cases in which the synthesis problem was infeasible. (b) pendulum angle and control signal recorded after 10 rollouts.}
	\label{fig:pendulum_costs}
\end{figure}

\subsubsection*{Acknowledgments}
This research was financially supported by the Swedish Foundation for Strategic Research (SSF) via the project \emph{ASSEMBLE} (contract number: RIT15-0012) and via the projects \emph{Learning flexible models for nonlinear dynamics} (contract number: 2017-03807) and \emph{NewLEADS - New Directions in Learning Dynamical Systems} (contract number: 621-2016-06079), both funded by the Swedish Research Council.

\medskip
\small


\clearpage

\appendix

\section{Supplementary material}


\subsection{Sampling from the posterior distribution}

\subsubsection{Case I: $\distcovtr$ known}\label{sec:inference_known}
First, consider the case where $\distcovtr$ is known; i.e., $\param=\groupparam{A,B}$.
From Bayes' rule, we have
\begin{equation}
\posterior(\param) \defeq p(\param|\data) 
=  \frac{1}{p(\data)}p(\data|\param)p(\param) 
\propto p(\param){\prod}_{i=1}^\numgendata p(x_{+}^i|x_{-}^i,u^i,\param)
\eqdef \uposterior(\param), 
\end{equation}
where $\data=\lbrace x_{+}^i,x_{-}^i,u^i\rbrace_{i=1}^\numgendata$ and $p(x_{+}^i|x_{-}^i,u^i,\param)=\normal{x_{+}^i-Ax_{-}^i-Bu^i}{\distcov}$.
When $\distcov$ is known, we can express the likelihood $p(\data|\param)$ in a form equivalent to an (un-normalized) Gaussian distribution over $\param$, i.e., 
\begin{subequations}
	\begin{align}
	&{\prod}_{i=1}^\numgendata p(x_{+}^i|x_{-}^i,u^i,\param)  \\
	&\propto \exp\left(-\frac{1}{2}\sum_i |\xplus^i-A\xminus^i-Bu^i|_{\distcov^{-1}}^2 \right) \\
	&=  \exp\left(-\frac{1}{2}\sum_i |\xplus^i-\xu^i\paramab|_{\distcov^{-1}}^2 \right) \\
	&= \exp\left(-\frac{1}{2}\left( \paramab'\left(\sum_i\xu^{i'}\distcov^{-1}\xu^i\right)\paramab  - 2\paramab'\sum_i\xu^{i'}\distcov^{-1}\xplus^i + \sum_i\xplus^{i'}\distcov^{-1}\xplus^i \right)  \right) \\
	&\propto \ \normal{\mu}{\Sigma},
	\end{align}
\end{subequations}
where $\Sigma= \left(\sum_i\xu^{i'}\distcov^{-1}\xu^i\right)^{-1}$, $\mu=\Sigma\left(\sum_i\xu^{i'}\distcov^{-1}\xplus^i\right)$, $\paramab = \textup{vec}\left(A';B'\right)$, and $\xu^i=I_{n_x}\otimes[\xminus^{i'} \ u^{i'}]$.
This implies that $\posterior(\param)=\normal{\mu}{\Sigma}p(\param)$.
Therefore, when the prior $p(\param)$ is non-informative ($p(\param)\propto1$) or Gaussian (self-conjugate), the posterior is also Gaussian.

\subsubsection{Case II: $\distcovtr$ unknown}\label{sec:inference_unknown}
Next, consider the generic case in which all parameters are unknown.
Then $\param=\groupparam{A,B,\distcov}$.
One approach to sampling from the posterior involves Gibbs sampling \cite{carter1994gibbs}, i.e., alternating between the following two sampling steps:
\begin{align}
\groupparam{A_k,B_k}&\sim p(A,B|\distcov_{k-1},\data), \\
\distcov_k &\sim p(\distcov|A_k,B_k,\data)
\end{align}
to form the Markov Chain $\groupparam{A_k,B_k,\distcov_k}_{k=1}^\infty$.
As demonstrated in \ref{sec:inference_known}, the distribution $p(A,B|\distcov_{k-1},\data)$ is Gaussian, so sampling is straightforward.
To sample from $p(\distcov|A_k,B_k,\data)$, first note
\begin{equation}
p(\distcov|A,B,\data)\propto p(\data|A,B,\distcov)p(\distcov).
\end{equation}
Observe that
\begin{align*}
p(\data|A,B,\distcov) \propto& \ \frac{1}{\det(\distcov)^{\frac{\numgendata}{2}}}\exp\left(-\frac{1}{2}\sum_i |\xplus^i-A\xminus^i-Bu^i|_{\distcov^{-1}}^2 \right) \\
=& \ \frac{1}{\det(\distcov)^{\frac{\numgendata}{2}}}\exp\left(-\frac{1}{2}\trace{\wishmat\distcov^{-1}} \right), \ \wishmat\defeq \sum_i (\xplus^i-A\xminus^i-Bu^i)(\xplus^i-A\xminus^i-Bu^i)' \\
\propto& \ \wishart{\wishmat}{\wishdof}, \wishdof = \numgendata - n_x - 1,
\end{align*} 
where $\wishart{\cdot}{\cdot}$ denotes the inverse Wishart distribution.
Note, if $\numgendata\leq n_x+1$ then $\wishdof$ is not valid.
However, we may consider a prior $p(\distcov)$ such as $p(\distcov)\propto\det(\distcov)^{-\frac{n_x+1}{2}}$ (Jeffreys' noninformative prior) which means
\begin{equation}
p(\distcov|A,B,\data)\propto \frac{1}{\det(\distcov)^{\frac{\numgendata+n_x+1}{2}}}\exp\left(-\frac{1}{2}\trace{\wishmat\distcov^{-1}} \right) \propto \wishart{\wishmat}{\wishdof},
\end{equation}
where $\wishdof = \numgendata > 0$.
This is a well-defined inverse Wishart distribution, sampling from which is straightforward.

\subsection{Proofs}

\subsubsection{Proof of Lemma \ref{lem:taylor}}\label{sec:taylor_proof}

\begin{lemma*}
	$\taylor(S,S_0) \mle S^{-1}$ for all $S,S_0\in\sym{n}_{++}$, where $\taylor(S,S_0)$ denotes the first-order Taylor series expansion of $S^{-1}$ about $S_0$. 
\end{lemma*}
\begin{proof}
	Let $D=S-S_0=L'L$, i.e, $L$ is the Cholesky factorization of $D$. Then,
	\begin{align*}
	S_0\mg0 & \iff
	S_0^{-1}\mg0 \implies
	LS_0^{-1}L' \mge 0 \iff 
	I+LS_0^{-1}L' \mge I \iff 
	(I+LS_0^{-1}L')^{-1}  \mle I \\ &\iff
	S_0^{-1}L'(I+LS_0^{-1}L')^{-1} LS_0^{-1} \mle S_0^{-1}L'LS_0^{-1} \\ & \iff 
	S_0^{-1} - S_0^{-1}L'(I+LS_0^{-1}L')^{-1} LS_0^{-1}  \mge S_0^{-1} - S_0^{-1}DS_0^{-1} \\ & \iff   
	\left(S_0 + L'L\right)^{-1}  \mge S_0^{-1} - S_0^{-1}\left(S-S_0\right)S_0^{-1} \iff  
	S^{-1}   \mge \taylor(S,S_0), 
	\end{align*}
	where the penultimate implication
	invokes
	the Woodbury matrix inversion identity \cite[eq. 159]{matrixcookbook}.
\end{proof}

\subsubsection{Proof of Theorem \ref{thm:bound}}\label{sec:bound_proof}
\begin{theorem*}
	Let $\mccostbound(K,\nomk)$ be defined as in \eqref{eq:bound}, \discuss{with $\nomk$ such that $\mccost(\nomk)$ is finite.} 
	Then $\mccostbound(K,\nomk)$ is a convex upper bound on $\mccost(K)$, i.e., $\mccostbound(K,\nomk)\geq\mccost(K) \ \forall K$.
	Furthermore, the bound is `tight' at $\nomk$, i.e., $\mccostbound(\nomk,\nomk)=\mccost(\nomk)$.
\end{theorem*}
\begin{proof}
		We will prove that $\costbound(K,\nomk\kpsep\param\ssc{i})$ is a tight convex bound on $\cost(K\kpsep\param\ssc{i})$, as this implies that
		$\mccostbound(K,\nomk)\defeq \frac{1}{\nummc}\sum_i\costbound(K,\nomk\kpsep\param\ssc{i})$ is a tight convex bound on $\mccost(K)\defeq\frac{1}{\nummc}\sum_i\cost(K\kpsep\param\ssc{i})$.
		First, we prove convexity.  
		$\costbound(K,\nomk\kpsep\param\ssc{i})$ is defined as the supremum over an infinite family of convex functions over a compact convex set, and is therefore a itself convex function.
		Note: the minimum of a linear function, e.g. $\min_{\lyap\ssc{i}\in\sympos{n_x}} \trace{\lyap\ssc{i}\bw\ssc{i}\bw\ssc{i}'} $  can be trivially expressed as the supremum of a convex function, i.e., $\sup_{\lyap\ssc{i}\in\sympos{n_x}} -\trace{\lyap\ssc{i}\bw\ssc{i}\bw\ssc{i}'} $.
		Next, we prove the upper bound.
		From Lemma \ref{lem:taylor}, $\lyap\ssc{i}^{-1}\mge \taylor_i(\lyap\ssc{i},\nomlyap\ssc{i})$ for all $\lyap\ssc{i}\in\sympos{n_x}$.
		Therefore, any $\lyap\ssc{i}\in\sympos{n_x}$ that satisfies \eqref{eq:bound_lmi} also satisfies \eqref{eq:localCost_nlmi}.
		This means the feasible set of \eqref{eq:localBound} is a subset of the feasible set of \eqref{eq:localCostAlt}, hence $\costbound(K,\nomk\kpsep\param\ssc{i})\geq\cost(K\kpsep\param\ssc{i})$.
		Finally, we prove tightness.
		As we have already proved $\costbound(K,\nomk\kpsep\param\ssc{i})\geq\cost(K\kpsep\param\ssc{i})$, it suffices to prove that $\nomx\ssc{i}$ is a feasible solution to \eqref{eq:bound_lmi}.
		As $\taylor(\nomx\ssc{i},\nomx\ssc{i})=\nomx\ssc{i}^{-1}$, for $K=\nomk$ and $\lyap\ssc{i}=\nomx\ssc{i}$ \eqref{eq:bound_lmi} is equivalent to \eqref{eq:localCost_nlmi}, which is feasible by definition of $\nomx\ssc{i}$.
		Hence, $\nomx\ssc{i}$ is a feasible solution of \eqref{eq:localBound}, that achieves $\trace{\nomx\ssc{i}\bw\ssc{i}\bw\ssc{i}'} = \cost(\nomk\kpsep\param\ssc{i})$, by definition of $\nomx\ssc{i}$.
\end{proof}

\subsubsection{Proof of Theorem \ref{thm:convex_hull}}\label{sec:convex_hull_proof}

\begin{theorem*}
	Suppose there exists $K\in\real^{n_x\times n_u}$ such that $(A\ssc{i}+B\ssc{i}K)'X(A\ssc{i}+B\ssc{i}K)-X\ml0$ for $X\mg0$ and all $\Theta=\lbrace A\ssc{i},B\ssc{i}\rbrace_{i=1}^N$.
	Then $(A+BK)'X(A+BK)-X\ml0$ for all $\lbrace A,B\rbrace\in\conv\Theta$, where $\conv\Theta$ denotes the convex hull of $\Theta$.
\end{theorem*}
\begin{proof}
	It is sufficient to show that $(A+BK)'X(A+BK)-X\ml0$ defines a convex set in terms of $(A,B)$. By the Schur complement,
	\begin{equation*}
	(A+BK)'X(A+BK)-X\ml0 \iff \left[\begin{array}{cc}
	X & (A+BK)' \\ A+BK & X^{-1}
	\end{array}\right]\mg 0,
	\end{equation*}
	which is convex in $A,B$ for given (i.e. fixed) $K$ and $X$.
\end{proof}

\subsection{Additional material for experiments on the rotary inverted pendulum}\label{sec:pendulum_details}

\subsubsection{System description}
The Quanser QUBE-Servo 2 inverted pendulum is depicted in Figure \ref{fig:pendulum_photo}.
The system consists of an actuated arm that rotates in the horizontal plane;
actuation is provided by an electrical motor.
Attached to the end of the rotary arm is an un-actuated pendulum, which is free to rotate.
The voltage applied to the electric motor constitutes the control input $u$ for the LQR problem.
Rotary encoders record the angular position of the rotary arm and pendulum, denoted $\arm$ and $\pend$, respectively.
These angular positions are fed through a high-pass-filter to provide angular velocity estimates, $(\dot{\arm},\dot{\pend})$.
The observed state is then given by $x=[\arm,\pend,\dot{\arm},\dot{\pend}]'$.

\subsubsection{Experimental procedure}\label{sec:rollout}
To generate one rollout of problem data, we first swing-up the pendulum to the inverted position, stabilized by an LQR designed using a physics-based model of the system.
Then we apply the voltage signal $v_t=Kx_t+\sin(\omega t) + w_t$ and record the resulting state evolution (sampled at 100Hz), until the pendulum angle $|\pend|$ exceeds $20^\circ$, or the rotary arm angle $|\arm|$ exceeds $50^\circ$. 
Here $K=[1,-10,1,-3]$ constitutes a state feedback policy that does \emph{not} stabilize the system, but does keep the pendulum upright from slightly longer than if it were absent.
This extends the typical rollout duration to around 3-5 seconds, before the pendulum angle exceeds $20^\circ$.
The angular frequency $\omega$ is randomly sampled each rollout, $\omega\sim\mathcal{U}[20,35]$, where $\mathcal{U}[a,b]$ denotes the uniform distribution over the interval $[a,b]$.
Finally, $w_t$ denotes band-limited white noise, with a sampling time of $T_s=0.01$, and a gain of 0.05.
$w_t$ represents additional exogenous disturbances that we artificially introduce to the system.

Training data then consists of $\lbrace x_t,u_t\rbrace_{t=0}^T$, where $x_t$ denotes the recorded state sequence, and $u_t =Kx_t+\sin(\omega t) $, i.e., the disturbance $w_t$ is not observed for learning.
$T$ is truncated to 500 (i.e. 5 seconds) in the event that the rollout lasts this long. 
A typical rollout is depicted in Figure \ref{fig:rollout}.

\begin{figure}[h]
	\centering
	\includegraphics[width=0.5\linewidth]{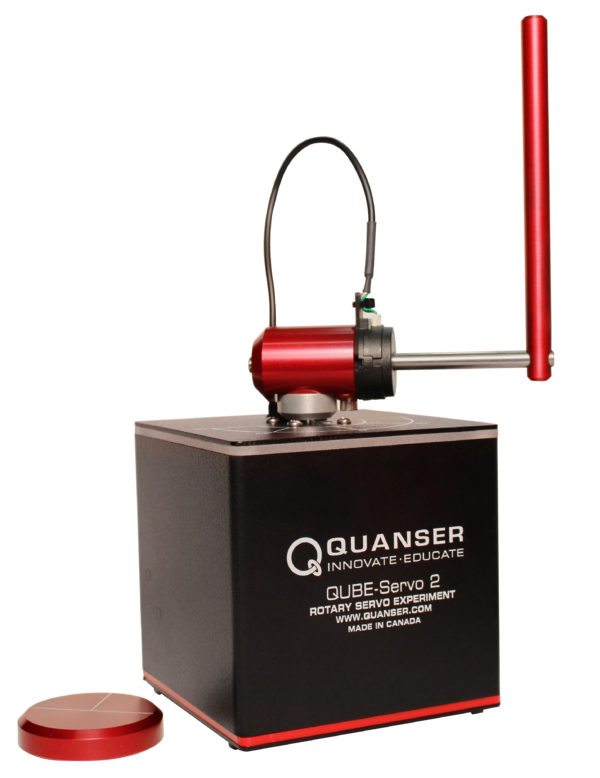}
	\caption{The Quanser QUBE-Servo 2 rotary pendulum, in the inverted position. Photo: www.quanser.com/products/qube-servo-2.}
	\label{fig:pendulum_photo}
\end{figure}

\begin{figure}[h]
	\centering
	\includegraphics[width=0.8\linewidth]{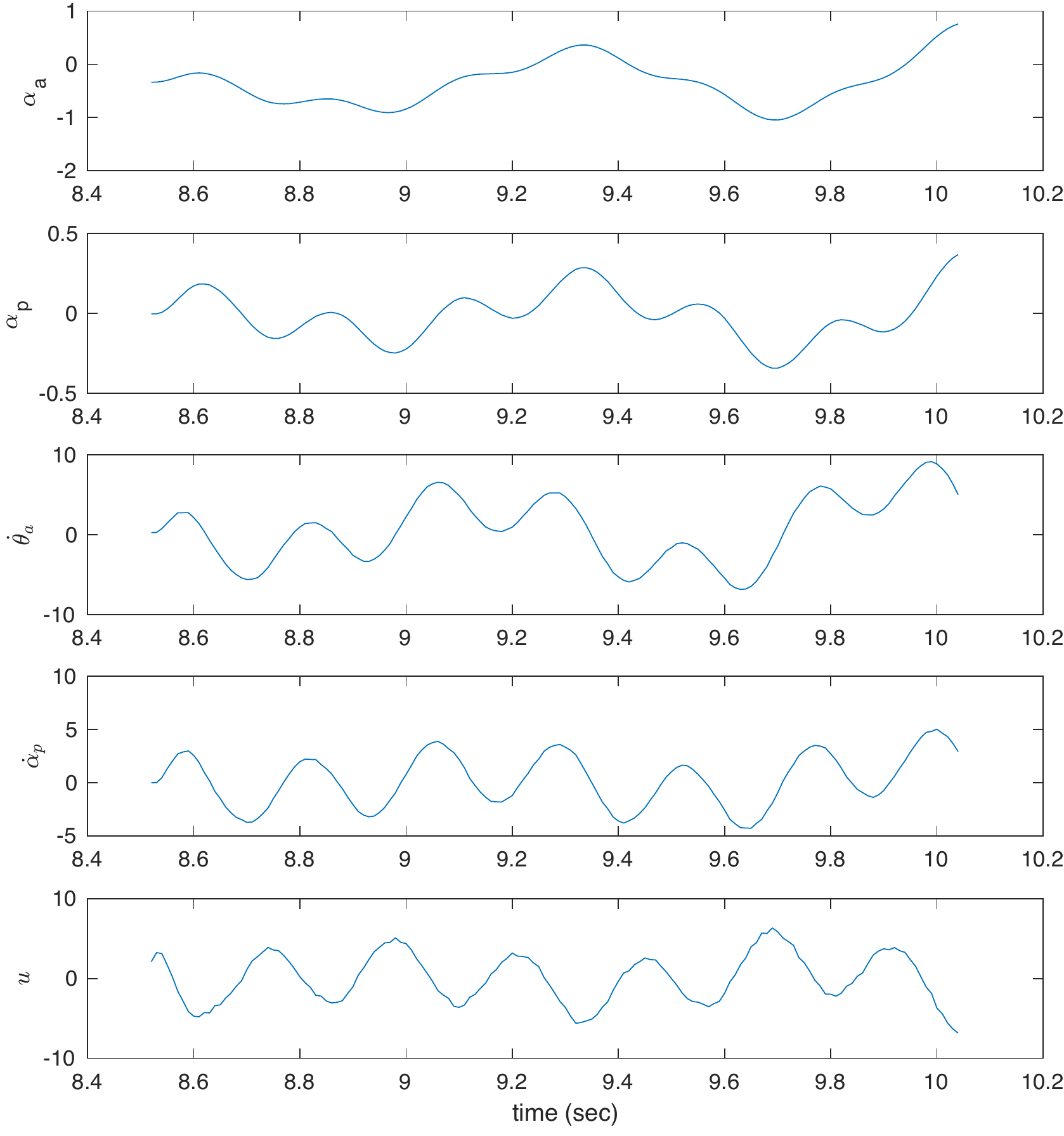}
	\caption{A typical rollout from the experimental procedure in \ref{sec:rollout}.}
	\label{fig:rollout}
\end{figure}

\subsubsection{Bounds for robust methods}
The robust synthesis methods \wc, \mixed, and \mixeditr\ require bounds on the error of the least squares estimate, i.e., $\epsilon_A\geq\norm{\als-\atr}{2}$ and $\epsilon_B\geq\norm{\bls-\btr}{2}$.
In \cite{dean2017sample}, these bounds are estimated, with a specific confidence level, via a Boostrap algorithm, assuming that the covariance is known.
In our setting, we estimate these bounds as described in \mysec~\ref{sec:pendulum}, i.e., by sampling from the 95\% confidence region of the posterior distribution.
This ensures a fair comparison between the methods, as they are, in essence, required to stabilize the same region of the parameter space.
We observed, however, that these bounds on the least squares error were too conservative; the magnitude of the uncertainty was too large, and the control synthesis optimization problems were infeasible.
The experiments presented in Figure \ref{fig:pendulum_costs} were attained by scaling down these error bounds by a factor of 100.
A number of scaling factors were tested, but 100 was found to achieve a reasonable trade-off between robustness and feasibility.
It is worth emphasizing that the \ours\ method used the samples from 95\% confidence region without any such scaling.

\end{document}